\documentclass[letterpaper]{article} 
\usepackage[submission]{aaai25}  
\usepackage{times}  
\usepackage{helvet}  
\usepackage{courier}  
\usepackage[hyphens]{url}  
\usepackage{graphicx} 
\usepackage{natbib}  
\usepackage{caption} 
\usepackage{algorithm}
\usepackage{algorithmic}

\usepackage{newfloat}
\usepackage{listings}
\DeclareCaptionStyle{ruled}{labelfont=normalfont,labelsep=colon,strut=off} 
\floatstyle{ruled}
\newfloat{listing}{tb}{lst}{}
\floatname{listing}{Listing}
\usepackage{hyperref}

\input{Defs/mathdefs}
\usepackage{amsthm}
\usepackage{amssymb}
\usepackage{thmtools}

\declaretheoremstyle[
	    spaceabove=\topsep, 
	    spacebelow=\topsep, 
	    headfont=\normalfont\bfseries,
	    bodyfont=\normalfont\itshape,
	    notefont=\normalfont\bfseries,
	    notebraces={(}{)},
	    postheadspace=0.33em, 
	    headpunct={.},
    ]{theorem}
\declaretheorem[style=theorem]{theorem}

\declaretheoremstyle[
	    spaceabove=\topsep, 
	    spacebelow=\topsep, 
	    headfont=\normalfont\bfseries,
	    bodyfont=\normalfont,
	    notefont=\normalfont\bfseries,
	    notebraces={(}{)},
	    postheadspace=0.33em, 
	    headpunct={.},
    ]{definition}

\declaretheoremstyle[
        spaceabove=\topsep, 
        spacebelow=\topsep, 
        headfont=\normalfont\bfseries,
        bodyfont=\normalfont,
        notefont=\normalfont\bfseries,
        notebraces={}{},
        postheadspace=0.33em, 
        qed=$\blacksquare$, 
        headpunct={.},
    ]{proofstyle}
\declaretheorem[style=proofstyle,numbered=no,name=Proof]{proof}

\declaretheorem[style=theorem,sibling=theorem,name=Lemma]{lemma}

\declaretheorem[style=theorem,numbered=no,name=Theorem]{theorem*}
\declaretheorem[style=theorem,numbered=no,name=Lemma]{lemma*}
\declaretheorem[style=theorem,numbered=no,name=Corollary]{corollary*}
\declaretheorem[style=theorem,numbered=no,name=Proposition]{proposition*}
\declaretheorem[style=theorem,numbered=no,name=Claim]{claim*}
\declaretheorem[style=theorem,numbered=no,name=Fact]{fact*}
\declaretheorem[style=theorem,numbered=no,name=Observation]{observation*}
\declaretheorem[style=theorem,numbered=no,name=Conjecture]{conjecture*}

\declaretheorem[style=definition,numbered=no,name=Definition]{definition*}
\declaretheorem[style=definition,numbered=no,name=Remark]{remark*}
\declaretheorem[style=definition,numbered=no,name=Example]{example*}
\declaretheorem[style=definition,numbered=no,name=Question]{question*}

\newcommand{\ber}{\mathrm{Ber}}

\newcommand{\aOpt}{a_\star}
\newcommand{\muOpt}{\mu_\star}
\newcommand{\nta}[1][t]{n_{#1,a}}
\newcommand{\ntaStar}[1][t]{n_{#1,\aOpt}}
\newcommand{\ntat}[1][t]{n_{#1,a_{#1}}}
\newcommand{\lt}[1][t]{\ell_{#1}}
\newcommand{\lna}[1][n]{\ell_{#1,a}}
\newcommand{\lnat}{\ell_{(\ntat),(a_t)}}

\newcommand{\regret}[1][t]{\mathrm{regret}_{#1}}

\newcommand{\nb}{l}
\newcommand{\batchIdxs}{\tau}

\newcommand{\muHat}{\hat{\mu}}

\newcommand{\barlna}[1][n]{\bar{\ell}_{#1,a}}
\newcommand{\muBar}{\bar{\mu}}

\usepackage{enumitem}
\usepackage{xcolor}
\usepackage{crossreftools}

\usepackage[capitalise]{cleveref}

\usepackage{xspace}

\title{Batch Ensemble for Variance Dependent Regret in Stochastic
Bandits}
\author{
    Asaf Cassel\textsuperscript{\rm 1},
    Orin Levy\textsuperscript{\rm 1},
    Yishay Mansour\textsuperscript{\rm 1,2}
}
\affiliations{
    \textsuperscript{\rm 1}Blavatnik School of Computer Science, Tel Aviv University\\
    \textsuperscript{\rm 2}Google Research, Tel-Aviv\\
    acassel@mail.tau.ac.il, 
    orinlevy@mail.tau.ac.il, 
    mansour.yishay@gmail.com
}

\usepackage{bibentry}

\begin{document}

\maketitle

\begin{abstract}%
      Efficiently trading off exploration and exploitation is one of the key challenges in online Reinforcement Learning (RL). Most works achieve this by carefully estimating the model uncertainty and following the so-called optimistic model. Inspired by practical ensemble methods, in this work we propose a simple and novel batch ensemble scheme that provably achieves near-optimal regret for stochastic Multi-Armed Bandits (MAB). Crucially, our algorithm has just a single parameter, namely the number of batches, and its value does not depend on distributional properties such as the scale and variance of the losses. We complement our theoretical results by demonstrating the effectiveness of our algorithm on synthetic benchmarks.
\end{abstract}

%






\section{Introduction}
Multi-Armed Bandits is a classic framework for sequential decision-making under uncertainty. In this setting, an agent repeatedly interacts with an environment by choosing from a set of $K$ actions (arms) and subsequently observing a loss signal associated with their choice. The loss signal associated with each arm is a sequence of i.i.d random variables whose mean is unknown to the agent. The agent's goal is to minimize their cumulative loss, which presents a classic exploration vs. exploitation dilemma, i.e., whether to exploit the current knowledge of the losses or to further explore seemingly sub-optimal actions that may turn out to be better. This trade-off is measured via a notion termed regret, which is the difference between the agent's performance and that of an oracle who knows the best arm and chooses it throughout the interaction.

Learning algorithms for this setting were extensively studied, and date back to Robbins' paper \cite{robbins1952some}. In particular, \cite{lai1985asymptotically} established that the regret in this problem is lower bounded by $\Omega(\log T)$, and there exist learning algorithms that achieve this regret by maximizing a confidence bound modification of the empirical mean. A non-asymptotic analysis was later provided by \cite{auer2002finite}. Subsequent works obtain bounds that depend on subtler properties of the arms by constructing more elaborate confidence bounds (see e.g., \citet{maillard2011finite} for empirical confidence bounds or UCB-KL). 

A commonality of nearly all past works is that they explicitly encode the distributional assumptions on the arms into the algorithm. For example, UCB \cite{auer2002finite} builds confidence bounds tailored to distributions bounded in $[0,1]$, UCB-V \cite{audibert2009exploration} refines these confidence bounds by incorporating a variance estimate, and KL-UCB \cite{maillard2011finite} explicitly uses the KL-divergence to estimate uncertainty and often assume a parametric class to reduce computational complexity. This explicit encoding can (1) be disadvantageous when the arms are misspecified; (2) require delicate parameter tuning; or (3) require prior knowledge of the distributions.

In this work, we show that constructing standard mean estimators from a simple batching scheme and combining them using a $\min$ operator, yields an optimistic mean estimator. Choosing greedily with respect to this estimator yields regret bounds that depend on the true concentration properties of the arm distributions.

\paragraph{Our Contributions.}
Our main contribution is a simple MAB algorithm, that does not need tuning of parameters and has low computational overhead. We show that for the Bernoulli r.v. our algorithm achieves an instance-dependent regret that depends on the variances of the arms. We show that our scheme extends to many other distributions, including, distributions that are either symmetric around the mean, have bounded support, or lower bounded variance. Crucially, this adaptation is purely in analysis and does not require modifying the algorithm.

Our scheme easily adapts to a distributed environment. Concretely, instead of explicitly constructing an optimistic mean estimator for each arm, our algorithm may be viewed as separate (distributed) naive bandit algorithms each receiving separate samples, computing the means of each arm, and outputting the best empirical arm together with its empirical mean. The final decision is made by following the decision of the bandit algorithm with the best empirical mean. This interpretation corresponds to practical methods such as \cite{osband2016deep,tennenholtz2022reinforcement}, which use ensembles to encourage exploration.
We complement our theoretical findings by running experiments on synthetic benchmarks, showing that our scheme achieves low regret compared to alternative algorithms.

\paragraph{Related work.}
Similar ideas of ensemble and bootstrapping methods have previously been studied.
\cite{ash2021anti} introduced anti-concentrated confidence bounds for efficiently approximating the elliptical bonus, using an ensemble of regressors.
\cite{osband2016deep} applied bootstrapped DQN in the Arcade Learning Environment and obtained improved learning speed and cumulative performance across most games.
\cite{osband2016generalization}
Present randomized least-squares value iteration (RLSVI) - an algorithm designed to explore and generalize via linearly parameterized value functions. Their results established that randomized value functions are a useful tool for efficient exploration along with effective generalization.
\cite{peer2021ensemble} present the Ensemble Bootstrapped Q-Learning (EBQL) algorithm, a natural extension of Double-Q-learning to ensembles that is bias-reduced. They analyze it both theoretically and empirically.


There are bootstrapping methods that add pseudo-rewards, sample the pseudo-rewards, and then run the MAB on the perturbed sequence. These include GIRO \cite{KvetonSVWLG19-GIRO} and PHE \cite{KvetonSGB19-PHE}, which have an instance-dependent regret bound for Bernoulli rewards (but they are not variance-dependent bounds) and Reboot \cite{WangYHC2020-Reboot} which handles Gaussian rewards.

Sub-sampling techniques, combined with a dueling approach, have been first proposed in BESA \cite{BaransiMM14-BESA} for two arms, and extended in RB-SDA \cite{BaudryKM20-RB-SDA} and SSMC \cite{chan2020multi} for one-parameter exponential distributions.

The most related work to ours is MARS \cite{KhorasaniW23-MARS}, where they generate optimistic estimates by sampling multiple random subsets and taking the maximum average reward. They show a regret bound for distributions that are continuous and symmetric around the mean. Our methodology can handle non-symmetric distributions and has better computational complexity.

\section{Preliminaries}
\paragraph{Problem setup.}
In a stochastic $K-$armed bandit, each arm $a \in [K]$ is associated with a loss sequence $\lna, (n \ge 1)$ of i.i.d Bernoulli random variables with parameter $\mu_a \in [0,1]$. (Note that $\mu_a=\EE\brk[s]{\lna}$, for any $n$.)
Let $\aOpt \in \argmin_{a \in [K]} \mu_a$ be an optimal action and $\muOpt = \mu_{\aOpt}=\min_{a \in [K]} \mu_a$ be the optimal value.
At each time step $t =1,2,\ldots$, an agent interacts with the bandit by choosing an arm $a_t \in [K]$ and subsequently observes the random loss $\lt = \lnat$ where 
\begin{align}
\label{eq:nta}
    \nta
    =
    \sum_{\tau=1}^t \indEvent{a_\tau = a}
\end{align}
is the number of times arm $a \in [K]$ was played up to time $t$ (inclusive).
The agent does not know the problem parameters and must learn them on the fly. We quantify its performance via the (pseudo) regret
\begin{align*}
    \regret
    =
    \sum_{\tau=1}^{t} [\mu_{a_\tau} - \mu_{\aOpt}]
    ,
\end{align*}
which measures the performance gap between the agent and the optimal policy that plays an optimal arm at each step.
We note that the restriction to Bernoulli arms is mostly for simplicity and we discuss extensions in \cref{sec:beyond-bernoulli}.

\paragraph{Deviation bounds.}
In what follows, we require the following fundamental properties.
Let $X_n \in [0,1],$ ($n \ge 1)$ be i.i.d random variables with $\EE X_n = \mu$ and  $\mathrm{Var}(X_n) = \sigma^2$. The following is a standard Bernstein inequality for bounded random variables. 
\begin{lemma}
\label{lemma:bernstein}
    Let $\bar{\mu} = \frac{1}{n} \sum_{n'=1}^{n} X_{n'}$. With probability at least $1 - \delta$
    \begin{align*}
        \bar{\mu}
        \ge
        \mu
        -
        \frac{2}{n}\log\frac{1}{\delta}
        -
        \sqrt{\frac{\sigma^2}{n} \log \frac{1}{\delta}}
        .
    \end{align*}
\end{lemma}

Next, recall that if $X_n \sim \mathrm{Ber}(\mu)$ then $\sum_{n'=1}^{n}X_{n'}$ is Binomial with parameters $n,\mu$.
\begin{lemma}[\citet{wiklund2023another}, Corollary 1]
\label{lemma:anti-concentration}
    If $\mu \le 1 - 1/n$ then
    \begin{align*}
        \Pr\brk*{
        \mathrm{Bin}(n, \mu)
        \le
        n \mu
        }
        \ge
        \frac14
        .
    \end{align*}
\end{lemma}


\section{Algorithm and Main Results}
At a high level, our Batch Ensemble algorithm works as follows.
It splits the samples of each arm into multiple batches.
For each batch, it computes an (almost) empirical average.
It then computes the minimum of those estimates, which is an optimistic estimator (recall that we are dealing with losses).
It then plays the action with the lowest estimate.

In \Cref{sec:optEst} we analyze the properties of our estimator. The crucial and non-standard property is that with high probability it is an underestimate of the true expected value. The second property is a somewhat standard concentration bound using Bernstein inequality (\Cref{lemma:bernstein}).

In \Cref{sec:Algo} we describe the Batch Ensemble algorithm, and state its performance guarantees (\Cref{theorem:high-prob-regret}). (The proof is deferred to \Cref{sec:proofs}.)
We then discuss a few implementation details, such as a distributed computing view of the algorithm, and the ability to have an \emph{any time} guarantee.

\subsection{An Optimistic Mean Estimator}
\label{sec:optEst}

Suppose we have observed $n \ge 0$ samples of an arm $a \in [K]$, i.e., $\lna[n'], n' \in [n]$. We build the following mean estimator. First, let, $\nb \ge 1$ be a batch number to be determined later. Next, we split the $n$ samples of arm $a$ into $\nb$ (near-)equal batches\footnote{We used a round-robin schedule to create the batches, but any non-adaptive scheme would work.}
\begin{align*}
    \batchIdxs_{n, a, \nb'}
    =
    \brk[c]*{n' : n' = \nb' + i \cdot \nb \le n, i \in \ZZ_{\ge 0}}
    ,
    \nb' \in [\nb]
    .
\end{align*}
Our batch ensemble estimator is 
\begin{align}
\label{eq:optimistic-mean-estimator}
    \muHat_{n,a}
    =
    \min_{\nb' \in [\nb]} \muHat_{n,a,\nb'}
    ,
    \text{ where }
    \muHat_{n,a,\nb'}
    =
    \sum_{n' \in \batchIdxs_{n, a, \nb'}} \frac{\lna[n']}{\abs{\batchIdxs_{n, a, \nb'}}+2}
    ,
\end{align}
with the convention that an empty sum is equal to $0$. The purpose of adding $2$ in the denominator will be made apparent in the proof of the following result, which establishes the optimistic property of our estimator.
\begin{lemma}
\label{lemma:mean-estimator-optimism}
Let $\delta \in [0,1]$ and $\nb \ge 1$.
Then for any $n \ge 0$
\begin{align*}
    &
    \Pr\brk*{\muHat_{n,a} \le \mu_a} \ge 1 - e^{-2\nb / 7}
    .
\end{align*}
\end{lemma}
Notice that choosing $\nb = (7/2)\log(1/\delta)$ in \cref{lemma:mean-estimator-optimism} gives the standard high probability optimistic guarantee.

\begin{proof}
Fix $\nb' \in [\nb]$ and $a \in [K]$.
We show that $\Pr\brk*{ \muHat_{n,a,\nb'} \le \mu_a } \ge 1/4$ for all $n \ge 0$. Let $\tau = \abs{\batchIdxs_{n,a,\nb'}}$ be the number of samples in the $\nb'-$th batch when arm $a$ has $n$ samples.
If $\tau = 0$, the claim holds trivially.
If $\tau = 1$ then for $\mu_a \ge 1/3$ the claim holds since $\muHat_{n,a,\nb'} \le 1/3$ and for $\mu_a < 1/3$ it holds with probability $1 - \mu_a \ge 2/3$ becuase $\muHat_{n,a,\nb'} \sim \ber(\mu_a)$.
Now, assume that $\tau \ge 2$. If $\mu_a \ge 1 - 1/\tau$ then we have that
\begin{align*}
    \muHat_{n,a,\nb'}
    \le
    \frac{\tau}{\tau+2}
    =
    1 - \frac{2}{\tau+2}
    \le
    1 - \frac{1}{\tau}
    \le
    \mu_a
    ,
\end{align*}
where the second to last inequality used that $\tau \ge 2$.
If $\mu_a \le 1 - 1/\tau$ then by \cref{lemma:anti-concentration}
\begin{align*}
    \Pr\brk*{
    \muHat_{n,a,\nb'}
    \le
    \mu_a
    }
    &
    \ge
    \Pr\brk*{
    \sum_{n' \in \batchIdxs_{n,a,\nb'}} \lna[n']
    \le
    \tau \mu_a
    }
    \ge
    \frac14
    .
\end{align*}
Now, since the $\muHat_{n,a,\nb'}$ are composed of different variables, they are jointly independent, thus we have
\begin{align*}
    \Pr\brk*{
    \muHat_{n,a} > \mu_a
    }
    &
    =
    \Pr\brk*{
    \muHat_{n,a,\nb'} > \mu_a
    ,
    \;
    \forall
    \nb' \in [\nb]
    }
    \\
    &
    =
    \prod_{\nb' \in [\nb]}
    \Pr\brk*{
    \muHat_{n,a,\nb'} > \mu_a
    }
    \\
    &
    \le
    \brk*{
    1 - 1/4
    }^\nb
    \le
    e^{-2\nb/7}
    .
    \qedhere
\end{align*}

\end{proof}

The following result describes the concentration of our mean estimator. The proof is a straightforward application of \cref{lemma:bernstein}.
\begin{lemma}
\label{lemma:mean-estimator-concentration}
Let $\delta \in [0,1]$ and $\nb \ge 1$. With probability at least $1-\delta$, simultaneously for all $n \in [T]$
\begin{align*}
    \mu_a - \muHat_{n,a}
    \le
    \frac{2}{(n / \nb) + 1}\log\frac{3T}{\delta}
    +
    \sqrt{\frac{\sigma_a^2}{(n / \nb) +1} \log \frac{T}{\delta}}
    .
\end{align*}
\end{lemma}
\begin{proof}
We use \cref{lemma:bernstein} together with a union bound to get that with probability at least $1 - \delta$, simultaneously for all $n \in [T], \nb' \in [\nb]$
\begin{align*}
    \sum_{n' \in \batchIdxs_{n,a,\nb'}} 
    \frac{\lna[n']}{\abs{\batchIdxs_{n,a,\nb'}}}
    \ge
    \mu_a
    -
    \frac{2}{\abs{\batchIdxs_{n, a, \nb'}}}\log\frac{T}{\delta}
    -
    \sqrt{\frac{\sigma_a^2}{\abs{\batchIdxs_{n, a, \nb'}}} \log \frac{T}{\delta}}
    ,
\end{align*}
Recalling the definition of $\muHat_{n,a,\nb'}$ in \cref{eq:optimistic-mean-estimator}, we conclude that
\begin{align*}
    &
    \muHat_{n,a,\nb'}
    =
    \frac{\abs{\batchIdxs_{n,a,\nb'}}}{\abs{\batchIdxs_{n,a,\nb'}}+2}
    \sum_{n' \in \batchIdxs_{n,a,\nb'}} \frac{\lna[n']}{\abs{\batchIdxs_{n,a,\nb'}}}
    \\
    &
    \ge
    \frac{\abs{\batchIdxs_{n,a,\nb'}}}{\abs{\batchIdxs_{n,a,\nb'}}+2}
    \brk[s]*{
    \mu_a
    -
    \frac{2}{\abs{\batchIdxs_{n,a,\nb'}}}\log\frac{T}{\delta}
    -
    \sqrt{\frac{\sigma_a^2}{\abs{\batchIdxs_{n,a,\nb'}}} \log \frac{T}{\delta}}
    }
    \\
    &
    \ge
    \mu_a
    -
    \frac{2}{\abs{\tau_{t,\nb'}}+2}\log\frac{3T}{\delta}
    -
    \sqrt{\frac{\sigma_a^2}{\abs{\tau_{t,\nb'}}+2} \log \frac{T}{\delta}}
    \\
    &
    \ge
    \mu_a
    -
    \frac{2}{(n / \nb) + 1}\log\frac{3T}{\delta}
    -
    \sqrt{\frac{\sigma_a^2}{(n / \nb) +1} \log \frac{T}{\delta}}
    ,
\end{align*}
where the second inequality used that $\mu_a \le 1$ and the third that $\abs{\batchIdxs_{n,a,\nb'}} \ge (n/\nb) - 1$.
\end{proof}

\subsection{The Batch Ensemble Algorithm}
\label{sec:Algo}

We present the Batch Ensemble algorithm in \cref{alg:batch-ensemble}. The algorithm receives as input a sequence representing the number of batches to use at each time step, builds a mean estimator as described in \cref{eq:optimistic-mean-estimator}, and chooses the arm with the most optimistic (i.e., minimal) estimate.
\begin{algorithm}
\caption{Batch Ensemble for MAB} \label{alg:batch-ensemble}
\begin{algorithmic}[1]
    \STATE \textbf{input}: 
        number of batches $\nb_t$ for all $t \ge 1$.
    
    \STATE \textbf{initialize}: 
        pull counts $\nta[0] = 0$ for all $a \in [K]$.
        
    \FOR{time step $t = 1,2,\ldots$}
    
        \STATE calculate $\muHat_{\nta[t-1], a}$ as in \cref{eq:optimistic-mean-estimator} with $\nb_t$ batches and choose
        \begin{align}
        \label{eq:decision-rule}
        a_t
        \in
        \argmin_{a \in [K]} \muHat_{\nta[t-1], a}
        .
        \end{align}

        \STATE observe $\lna[(\nta)]$ and update $\nta = \nta[t-1] + \indEvent{a_t = a}$.
    \ENDFOR
\end{algorithmic}
\end{algorithm}

The following is our main result, which bounds the regret of the above algorithm (proof in \cref{sec:proofs}).
\begin{theorem}
\label{theorem:high-prob-regret}
    Suppose we run \cref{alg:batch-ensemble} with a fixed number of batches with $\nb = (7/2) \log (2T / \delta)$. With probability at least $1-\delta$ the following regret bounds hold simultaneously
    \begin{align*}
        \regret[T]
        &
        \le
        \frac{7}{2}
        \sum_{a \neq \aOpt} \brk*{
            \frac{\sigma_a^2}{\Delta_a}
            +
            2
        }\log^2\frac{6TK}{\delta}
        \\
        &
        \le
        \frac{7}{2}
        \sum_{a \neq \aOpt} \brk*{
            \frac{\muOpt}{\Delta_a}
            +
            3
        }\log^2\frac{6TK}{\delta}
        \\
        \regret[T]
        &
        \le
        \sqrt{14 T \min\brk[c]*{\muOpt K, \sum_{a \neq \aOpt}\sigma_a^2}} \log\frac{6TK}{\delta}
        \\
        &
        +
        11 K \log^2\frac{6TK}{\delta}
        .
    \end{align*}
\end{theorem}
In the above, we use the fact that $\sigma_a^2\leq \mu_a=\muOpt+\Delta_a$.
\paragraph{Dependence on the true arm distributions.}
We note that our regret bounds depend on the variance due to our use of the Bernstein type inequality in \cref{lemma:bernstein}. This choice was made in order to obtain a clear and intuitive result. A tighter bound can be achieved by using the tight Chernoff bound for the empirical mean of each arm, which for Bernoulli r.v depends on the KL divergence. This is true even if different arms come from different distributional families, e.g., some are normal and some are Bernoulli, and does not affect the algorithm or its parameter.

%


\paragraph{A distributed view.}
We presented \cref{alg:batch-ensemble} in the familiar UCB-style index rule. However, a potentially more insightful perspective is to view each batch $\nb' \in [\nb]$ as a separate bandit algorithm that at each round $t \in [T]$ outputs a prediction 
$
    a_{t, \nb'}
    \in
    \argmin_{a \in [K]} \muHat_{(\nta[t-1]), (a), (\nb')}
$
and its estimate  
$
    \muHat_{t, \nb'}
    =
    \muHat_{(\nta[t-1]), (a_{t, \nb'}), (\nb')}
    ,
$
and the final decision is made by greedily choosing the best batch, i.e., $a_t = a_{t, \nb_t^\star}$ where 
$
    \nb_t^\star
    \in
    \argmin_{\nb' \in [\nb]} \muHat_{t,\nb'}
    .
$
This decision rule is equivalent to \cref{eq:decision-rule}.
Importantly, we believe this view could help scale our approach to other problem settings such as MDPs where each batch would output a policy and its value prediction and the final policy is the one associated with the most optimistic batch. We leave this for future research.

We note that, unlike some distributed learning routines that aggregate decisions by averaging them to reduce uncertainty (noise), our approach selects a single ``noisy'' batch whose decision is followed. This is key to ensuring the optimistic property of our algorithm.


\paragraph{An anytime expected regret algorithm.}
One can always use the doubling trick to obtain an anytime algorithm. However, this leaves the dependence on the confidence level $\delta$ that, realistically, depends on the time horizon $T$. To avoid this, we show that choosing the number of batches as $\nb_t = 8 \log t$, \cref{alg:batch-ensemble} has expected regret ($\EE \regret, t \ge 1$) bounded similarly to \cref{theorem:high-prob-regret} but with $T$ replaced with $t$ and the dependence on $\delta$ removed. For details see the supplementary material.

\section{Beyond Bernoulli arms}
\label{sec:beyond-bernoulli}

Our results thus far focused on Bernoulli distributed arms. However, this is not explicitly encoded in our algorithm but rather in its analysis.
In fact, our algorithm works without change for any arm distributions that satisfy properties akin to \cref{lemma:bernstein,lemma:anti-concentration}. 

\cref{lemma:bernstein} is a standard concentration bound for bounded random variables. If the random variables are unbounded, \cref{lemma:bernstein} can be replaced with the appropriate Chernoff bound. As long as the distributions are light-tailed (e.g., sub-Gaussian or sub-exponential), this will not change the regret bound significantly. We emphasize that the algorithm does not need to know the tail behavior and thus the bound may be tailored to the true distribution of each arm.

\cref{lemma:anti-concentration} is a type of anti-concentration result for sums of Bernoulli random variables. We conjecture that all bounded random variables satisfy this property, but have been unable to prove this. For further details see the conjecture at the end of the section.
In what follows, we describe several methods and conditions to satisfy the anti-concentration property for non-Bernoulli arms.
\paragraph{Bernoulli-fication.}
It is well known that arm distributions in $[0,1]$ can be converted into Bernoulli arms. To do this one replaces the observed losses of the algorithm $\lna$ with samples $\barlna \sim \ber(\lna)$. If the distribution is in $[0,b]$, one can first scale the losses by dividing with $b$. It is straightforward to verify that $\EE \lna = \EE \barlna$, and thus \cref{theorem:high-prob-regret} holds but with $\sigma_a$ replaced with the variance of $\barlna$. This does not impact the first-order regret bounds, but can significantly increase the variance-dependent (second-order) bounds (e.g., for deterministic arms).

\paragraph{Scaled Bernoulli.}
It is often the case that arm distributions are not evenly scaled. Most bandit algorithms such as UCB or UCB-V have a single scale parameter, which bounds the worst-case arm. \cref{alg:batch-ensemble} does not need to know the scales in advance and automatically enjoys dependence on the true arm scales. To see this, consider arm distributions that are scaled Bernoulli variables with parameters $b_a \ge \mu_a \ge 0$ such that
\begin{align*}
    \lna
    =
    \begin{cases}
        b_a, &\text{ w.p } \mu_a/b_a
        \\
        0, &\text{ otherwise}
        .
    \end{cases}
\end{align*}
Notice that if we scale the arms in the analysis, we can still use \cref{lemma:anti-concentration} to get the optimism claim in \cref{lemma:mean-estimator-optimism}. As for concentration, scaling \cref{lemma:bernstein} replaces the $2/n$ term with $2b_a/n$. Propagating this into the analysis of \cref{theorem:high-prob-regret} would modify the bounds such that
\begin{align*}
    \brk[s]*{
        \frac{\sigma_a^2}{\Delta_a}
        +
        2
    }
    ,
    \brk[s]*{
        \frac{\muOpt}{\Delta_a}
        +
        3
    }
    &\implies
    \brk[s]*{
        \frac{\sigma_a^2}{\Delta_a}
        +
        2b_a
    }
    ,
    \brk[s]*{
        \frac{\muOpt}{\Delta_a}
        +
        3b_a
    }
    \\
    11 K \log^2\frac{6TK}{\delta}
    &\implies
    11 \sum_{a \neq \aOpt} b_a \log^2\frac{6TK}{\delta}
    .
\end{align*}
Notice that we do not depend on the scale of the optimal arm, which could be meaningful when it is significantly larger than the scale of sub-optimal arms.

\paragraph{Symmetric distributions}
Our algorithm works unchanged for symmetric arm distributions (around their mean).
To see this, notice that the sum of symmetric random variables is also symmetric, and thus \cref{lemma:anti-concentration} holds with probability at least $1/2$ (instead of $1/4$).
We note that unlike \cite{khorasani2023maximum}, we do not require the distributions to be continuous.
In particular, the above implies that our algorithm works for Gaussian arm distributions.

\paragraph{Arms with lower bounded variance.}
Recall that the Central Limit Theorem (CLT) implies that any (appropriately scaled) sum of random variables converges in distribution to a Gaussian, which satisfies \cref{lemma:anti-concentration}. Concretely, this implies that even for non-symmetric random variables, the sample mean becomes symmetric as the sample size increases.
In the following, we make this informal argument concrete. Let $X_i, i \ge 1$ be i.i.d random variables with mean $\mu$. Let $C_\sigma \ge 0$ be a constant such that $\rho / \sigma^3 \le C_\sigma$ where $\sigma^2$ is the variance of $X_i$ and $\rho$ is its third absolute central moment
\begin{align*}
    \rho
    =
    \EE\brk[s]{\abs{X_i - \mu}^3}
    .
\end{align*}
Note that for any sufficiently light-tailed distribution (normal, exponential, bounded), $C_\sigma$ is bounded (up to a numerical constant) by $\sigma^{-1}$. Thus, it suffices to have a lower bound on the variance to bound $C_\sigma$.
Define $Y_n = (\sum_{i \in [n]} (X_i - \mu) / (\sigma \sqrt{n})$ and let $F_n(\cdot)$ be its Cumulative Distribution Function (CDF). 
The Berry-Esseen Theorem (see e.g. \cite{shevtsova2011absolute}) states that
\begin{align*}
    \abs{F_n(x) - \Phi(x)} \le C_\sigma / 2\sqrt{n}
    ,
    \quad
    \forall x \in \RR, n \ge 1
    ,
\end{align*}
where $\Phi$ is the CDF of the standard normal distribution. Taking $x = 0, n \ge 4 C_\sigma^2$ we conclude that
\begin{align*}
    \Pr\brk*{
    \sum_{i \in [n]} X_i \le n \mu
    }
    &
    =
    \Pr\brk{Y_n \le 0}
    \\
    &
    \ge
    \Phi(0) - C_\sigma / 2\sqrt{n}
    \ge 
    1/4
    ,
\end{align*}
which is the equivalent of \cref{lemma:anti-concentration}. We conclude that \cref{alg:batch-ensemble} can work for general distributions by adding a warmup phase that collects $4 C_\sigma^2 \approx 4 \sigma^{-2}$ samples for each arm and batch.

\paragraph{Conjecture.}
Notice that for Bernoulli arms we have $\sigma_a^2 = \mu_a (1-\mu_a)$. As such, the above logic would suggest that a long warmup phase is necessary when $\mu_a$ is close to either $0$ or $1$. However, \cref{lemma:anti-concentration} reveals this to be unnecessary. The reason for this gap is that we only require $F_n(0)$ to be sufficiently large whereas the Berry-Esseen Theorem ensures $F_n(x)$ converges to $\Phi(x)$ for all $x$, which is much stronger.

Notice that the Bernoulli distribution can be extremely asymmetric when $\mu_a$ is close to $0$ or $1$. This leads us to believe that it might be the worst-case distribution for the anti-concentration result in   \cref{lemma:anti-concentration} (among the class of bounded random variables). Thus, we conjecture that \cref{lemma:anti-concentration} holds for any bounded arm distribution and consequently so do the regret guarantees of \cref{alg:batch-ensemble}.

\section{Proof of Theorem \ref{theorem:high-prob-regret}}
\label{sec:proofs}
Recall that the pseudo-regret may be written as
\begin{align}
\label{eq:reg-decomp}
    \regret[T]
    =
    \sum_{a \neq \aOpt} \nta[T] \Delta_a
    ,
\end{align}
where $\Delta_a = \mu_a - \muOpt$ is the optimality gap of arm $a \in [K]$ and $\nta$ is defined in \cref{eq:nta}. Thus, our goal is to bound $\nta[T]$ for each sub-optimal arm.
We begin with a standard ``good event'' over which the regret is bounded deterministically. Suppose that for all $n \in [T]$ and $a \neq \aOpt$ we have
\begin{align}
    &
    \label{eq:good-optimism}
    \muHat_{n,\aOpt} \le \muOpt
    \\
    &
    \label{eq:good-concentration}
    \muHat_{n,a}
    \ge
    \mu_a
    -
    \frac{2}{(n/\nb) + 1}\log\frac{6KT}{\delta}
    -
    \sqrt{\frac{\sigma_a^2}{(n/\nb) +1} \log \frac{2TK}{\delta}}
    .
\end{align}
Taking a union bound over \cref{lemma:mean-estimator-optimism} with $\delta/2$ and \cref{lemma:mean-estimator-concentration} with $\delta/2K$, the above holds with probability at least $1 - \delta$.
Now, suppose that arm $a$ was played at time $t$. Then by the decision rule in \cref{eq:decision-rule}, we have that
$
    \muHat_{\nta[t - 1], \aOpt} \ge \muHat_{\nta[t - 1],a}
    ,
$
and thus
\begin{align*}
    \muOpt
    \tag{\cref{eq:good-optimism}}
    &
    \ge
    \muHat_{\ntaStar[t-1],\aOpt}
    \\
    &
    \ge
    \muHat_{\nta[t-1], a}
    \\
    &
    \tag{\cref{eq:good-concentration}}
    \ge
    \mu_a
    -
    \frac{2}{(\nta[t-1] / \nb) + 1}\log\frac{6TK}{\delta}
    \\
    &
    \qquad
    -
    \sqrt{\frac{\sigma_a^2}{(\nta[t-1] / \nb) + 1} \log \frac{2TK}{\delta}}
    .
\end{align*}
Solving this quadratic inequality for $\nta[t-1]$, we have
\begin{align*}
    \nta[t-1]
    \le
    \nb \brk[s]*{
        -1
        +
        \brk*{
        \frac{\sigma_a^2}{\Delta_a^2}
        +
        \frac{2}{\Delta_a}
        }\log\frac{6TK}{\delta}
    }
    ,
\end{align*}
Now, let $t_a$ be the last time arm $a$ was chosen. Then we have
\begin{align*}
    \nta[T]
    =
    \nta[t_a]
    =
    1 + \nta[t_a - 1]
    &
    \le
    \frac{7}{2}
    \brk*{
        \frac{\sigma_a^2}{\Delta_a^2}
        +
        \frac{2}{\Delta_a}
    }\log^2\frac{6TK}{\delta}
    \\
    &
    \le
    \frac{7}{2}
    \brk*{
        \frac{\mu_a}{\Delta_a^2}
        +
        \frac{2}{\Delta_a}
    }\log^2\frac{6TK}{\delta}
    \\
    &
    =
    \frac{7}{2}
    \brk*{
        \frac{\muOpt}{\Delta_a^2}
        +
        \frac{3}{\Delta_a}
    }\log^2\frac{6TK}{\delta}
    ,
\end{align*}
where the first inequality used our choice of $\nb = (7/2)\log(2T/\delta)$ and the second inequality used that for any random variable in $[0,1]$, we have $\sigma^2 \le \mu$.
Plugging this into \cref{eq:reg-decomp} concludes the instance-dependent regret bounds.

%
%

Next, for instance-independent bounds, we use the standard method of splitting the bound according to the sub-optimality to get that for any $c > 0$,
\begin{align*}
    \nta[T] \Delta_a
    &
    =
    \nta[T] \Delta_a \brk[s]*{
        \indEvent{\Delta_a \le c^{-1}}
        +
        \indEvent{\Delta_a^{-1} < c}
    }
    \\
    &
    \le
    \frac{\nta[T]}{c}
    +
    \frac{7}{2}
    \brk*{
        \muOpt c
        +
        3
    }\log^2\frac{6TK}{\delta}
    .
\end{align*}
Plugging into \cref{eq:reg-decomp} and setting $c = \sqrt{\frac{2T}{7 \muOpt K \log^2(6TK/\delta)}}$ we have
\begin{align*}
    \regret[T]
    &
    \le
    \sum_{a \neq \aOpt}
    \frac{\nta[T]}{c}
    +
    \frac{7}{2}
    \brk*{
        \muOpt c
        +
        3
    }\log^2\frac{6TK}{\delta}
    \\
    &
    \le
    \frac{T}{c}
    +
    \frac{7}{2}K
    \brk*{
        \muOpt c
        +
        3
    }\log^2\frac{6TK}{\delta}
    \\
    &
    \le
    \sqrt{14\muOpt T K} \log\frac{6TK}{\delta}
    +
    11 K \log^2\frac{6TK}{\delta}
    .
\end{align*}
Finally, we perform a similar procedure for the variance-dependent bound to get that
\begin{align*}
    \nta[T] \Delta_a
    \le
    \frac{\nta[T]}{c}
    +
    \frac{7}{2}
    \brk*{
        \sigma_a^2 c
        +
        2
    }\log^2\frac{6TK}{\delta}
    ,
\end{align*}
and thus setting $c = \sqrt{\frac{2T}{7 (\sum_{a \neq \aOpt} \sigma_a^2) \log^2(6TK/\delta)}}$, we have
\begin{align*}
    \regret[T]
    &
    \le
    \sum_{a \neq \aOpt}
    \frac{\nta[T]}{c}
    +
    \frac{7}{2}
    \brk*{
        \sigma_a^2 c
        +
        2
    }\log^2\frac{6TK}{\delta}
    \\
    &
    \le
    \frac{T}{c}
    +
    \frac{7}{2}
    \brk*{
        c \brk*{\sum_{a \neq \aOpt}\sigma_a^2}
        +
        2K
    }\log^2\frac{6TK}{\delta}
    \\
    &
    \le
    \sqrt{14 T \sum_{a \neq \aOpt}\sigma_a^2} \log\frac{6TK}{\delta}
    +
    7 K \log^2\frac{6TK}{\delta}
    .
\end{align*}

\begin{figure}[t]
    \centering
    \includegraphics[width=\linewidth]{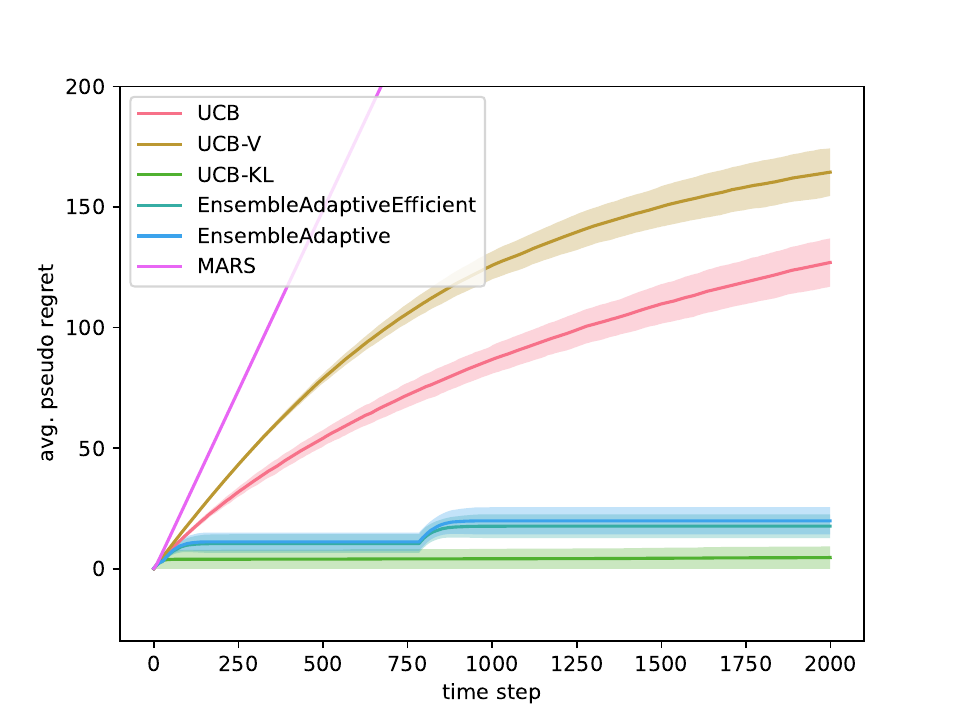}
    \caption{Results for $5$ Bernoulli arms with clear, low-variance, best arm. The arms expectations tested are $0.001,0.15,0.2,0.25,0.3$.}
    \label{fig:test1}
\end{figure}

\section{Experiments}
In this section, we compare the performance of our ensemble method to other provable methods such as \texttt{UCB}, \texttt{UCB-V}, \texttt{UCB-KL}, and \texttt{MARS}, on synthetic Bernoulli, Gaussian, and exponential MAB environments.
In our experiments, we refer to \cref{alg:batch-ensemble} as \texttt{EnsembleAdaptive}, and to a variant of it, \texttt{EnsembleAdaptiveEfficient} that has an improved running time. 
We note that the only difference between the two implementations is that, when increasing the number of batches, instead of redistributing all the samples, the more efficient implementation adds samples to the new (empty) batch until it reaches the size of the remaining batches.
In the following, we consider five test cases for MAB environments with five and ten arms. In all test cases, we run $100$ simulations each with $T=2000$ steps. The performance criterion is the averaged pseudo-regret across all simulations.

\noindent\textbf{Test case 1: Bernoulli arms with clear, low-variance, best arm.}
In this test case, we generated a synthetic environment of MAB with five Bernoulli arms, where the means are $0.001, 0.15, 0.2, 0.25,$ and $0.3$ for each arm respectively. Since we work with losses, the first arm is optimal (and has low variance). The other arms are non-optimal and the suboptimality gaps are relatively large, making best-arm identification easier, and demonstrating the superiority of variance-dependent methods.

As seen in \cref{fig:test1}, MARS has a long burn-in period with near-linear pseudo regret. 
On the other hand, the UCB-based algorithms obtain much better regret, with UCB-KL obtaining the lowest regret among all algorithms. However, our Ensemble method implementations are very close to UCB-KL, and significantly better than UCB, UCB-V, and MARS. 
We note that there is a slight jump in the plot of our Ensemble methods results due to the adaptive batch size, which increases the number of batches after approximately $800$ steps.

\noindent\textbf{Test case 2: Bernoulli arms with high expected loss and low variance.}
In this test case, we examine the performance of the algorithms in a high-expected loss scenario where the means are  $0.9,0.91,0.92,\ldots,0.99$. 
We note that in this case, identifying the best arm, i.e., the first arm, is hard, since the sub-optimality gaps are relatively low, with the smallest being $0.01$. The theory (specifically \cref{lemma:anti-concentration}) suggests this case could challenge our approach, thus making it interesting.
As arm variances are low, the variance-dependent algorithms obtain much better results than UCB and MARS. UCB-V is slightly better than UCB,  and the best performance is obtained by UCB-KL and our Ensemble methods, which have similar results.
\begin{figure}
    \centering
    \includegraphics[width=\linewidth]{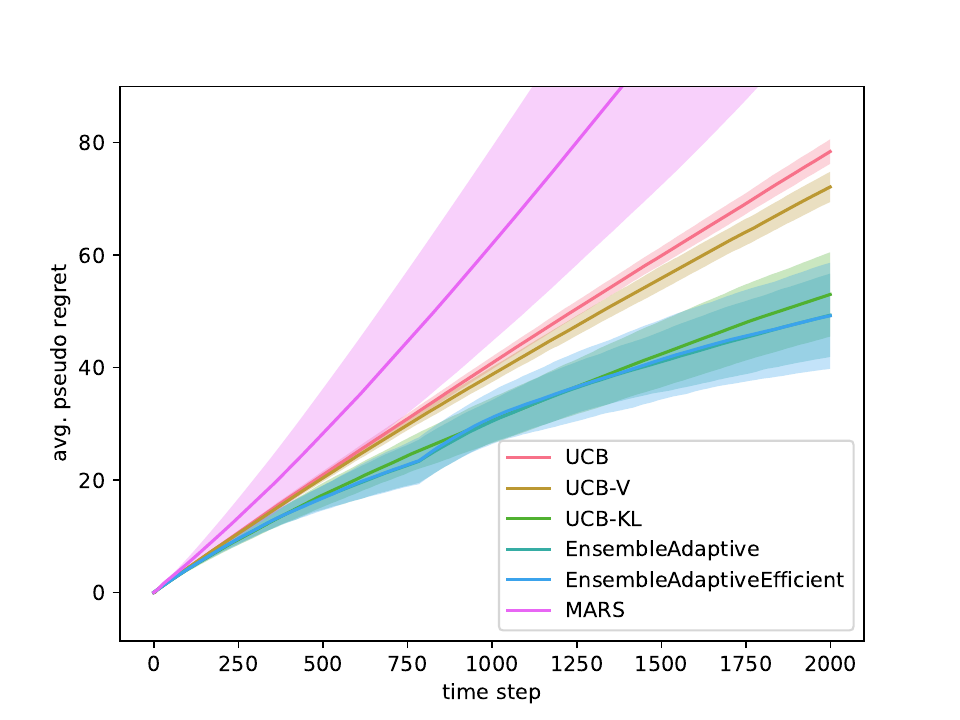}
    \caption{Results for $10$ Bernoulli arms with means $0.9,0.91,0.92,\ldots,0.99$.}
    \label{fig:test2}
\end{figure}

\noindent\textbf{Test case 3: Bernoulli arms with random means.}
This case examines the typical behavior of the algorithms for the classical case of randomly chosen means for Bernoulli arms. In this experiment, in each one of the $100$ simulations, we sampled $10$ numbers uniformly from the interval $[0,1]$. Each number represents the mean of one Bernoulli arm. We tested all algorithms using the same sampled means. 
\begin{figure}
    \centering
    \includegraphics[width=\linewidth]{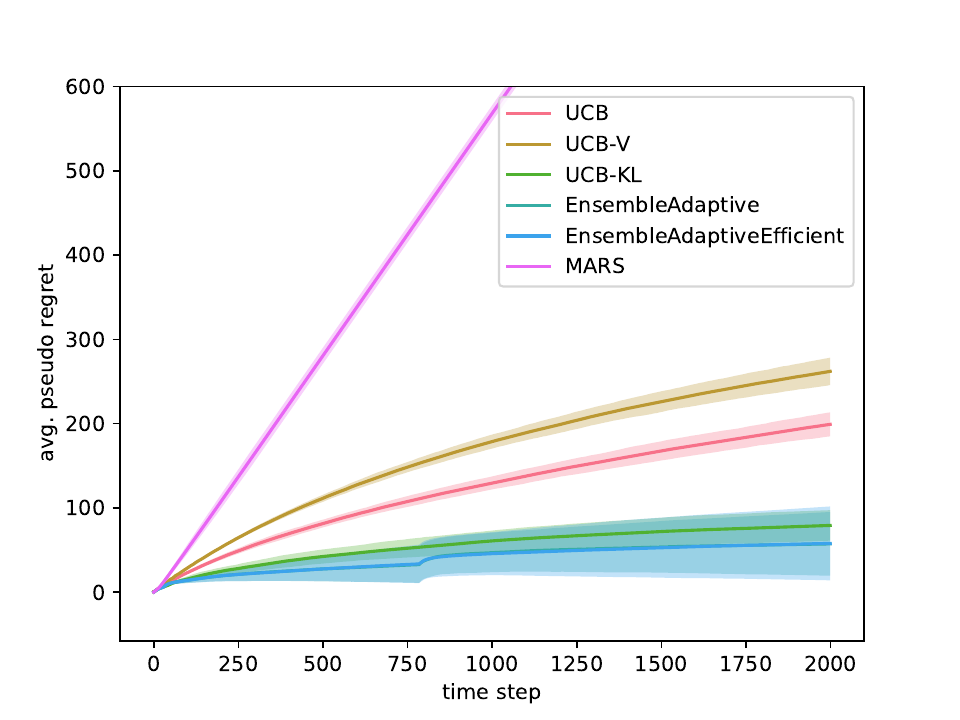}
    \caption{Results for $10$ Bernoulli arms with random means.}
    \label{fig:test3}
\end{figure}
\cref{fig:test3} presents similar trends. While MARS incurs near-linear regret, UCB-KL and our ensemble implementations perform similarly (with the ensembles having lower means but higher variance), and significantly better than UCB-V and UCB. Still, it is interesting to observe that the more efficient implementation obtains better results.
Also, in this case, UCB-V has a lower performance than UCB. This could be attributed to UCB having slightly better-tuned confidence bounds when the arms have a high variance.

\noindent\textbf{Test case 4: Gaussian arms with random means and variance 1.}
We also consider the case of $10$ Gaussian arms with randomly chosen means in $[0,1]$ and standard deviation $\sigma = 1$.
In each one of the $100$ simulations, we sampled uniformly at random $10$ numbers in $[0,1]$. Each chosen number represents the mean of the related Normal arm. We tested all the compared algorithms using the same means in all simulations.   
\cref{fig:test4} demonstrates similar trends to those observed in test case 2 but in a higher-variance environment. Again, MARS has near-linear regret, and UCB-V is beaten by UCB, likely due to the high variance of the arms. However, here, our Ensemble method implementations outperform UCB-KL. Another interesting observation is that UCB-KL and our ensemble methods have a relatively high standard deviation in this experiment.
\begin{figure}
    \centering
    \includegraphics[width=\linewidth]{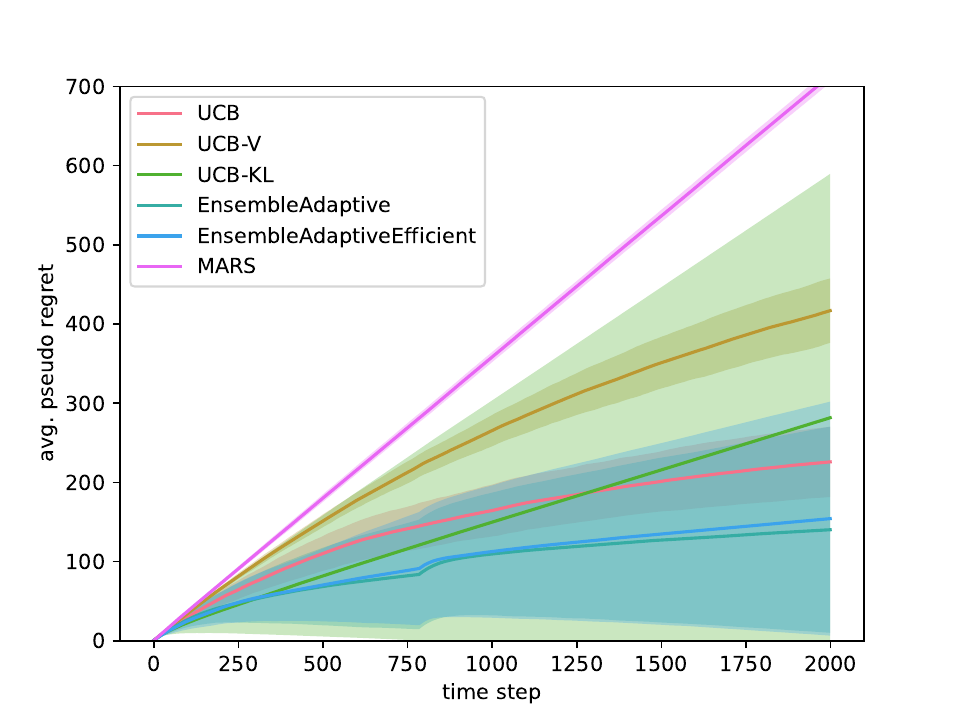}
    \caption{Results for $10$ Gaussian arms with random means and the variance 1.}
    \label{fig:test4}
\end{figure}

\vspace{1em}
\noindent\textbf{Test case 5: Exponential arms with random scales.}
Finally, consider MAB with $10$ Exponential arms and randomly chosen scales in $[0,1]$. This represents a scenario where our theoretical guarantees do not hold without a warmup, thus challenging our conjecture that our results hold even for non-symmetric distributions. In each one of the $100$ simulations, we sampled uniformly at random $10$ numbers in $[0,1]$. Each chosen number represents the scale of one Exponential arm. We tested all the compared algorithms using the same scales in each simulation, where we recall that for scale $\lambda>0$, the mean of the exponential distribution is $\lambda^{-1}$. Hence, the expected losses are relatively high, which is also a scenario that may challenge our algorithm.
\cref{fig:test5} is a positive signal for our conjecture. Our ensemble implementations outperform UCB-KL. Standard UCB performs worse but outperforms UCB-V (due to relatively high variance arms), and MARS again has near-linear regret albeit with performance ranging between UCB and UCB-V.

\begin{figure}
    \centering
    \includegraphics[width=\linewidth]{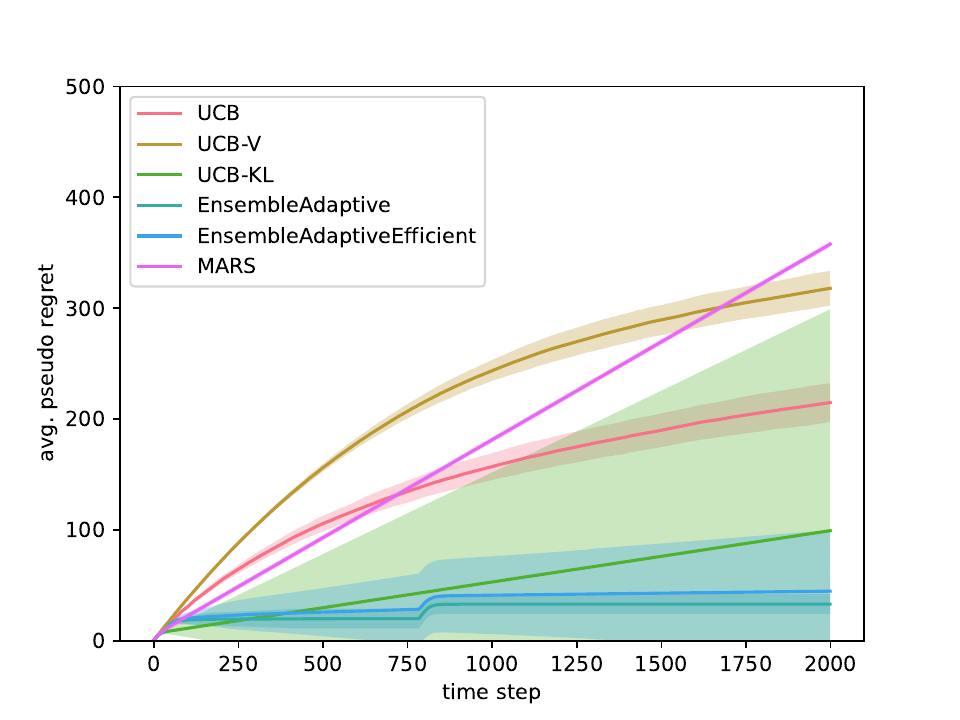}
    \caption{ $10$ Exponential arms with random scales.}
    \label{fig:test5}
\end{figure}

\noindent\textbf{Discussion and summary of the results.}
Our experiments point to our ensemble implementations having good performance across several scenarios. Our performance is close and sometimes better compared to UCB-KL and consistently better than the remaining methods, especially in scenarios where the arms have low variances.
Another advantage of our ensemble algorithms is their running time, which is close to that of standard UCB (constant per step), and memory usage, which scales logarithmically in the time horizon $T$.
In comparison, MARS has a running time and memory usage of $O(T^2)$, and for Bernoulli arms, UCB-KL has to solve an optimization problem using an interior point algorithm at each time step.
Overall, this positions our ensemble implementations as promising practical methods. As an added bonus, our improved efficiency variant exhibits similar and sometimes improved performance. 

\section*{Acknowledgements}
This project has received funding from the European Research Council (ERC) under the European Union’s Horizon
2020 research and innovation program (grant agreement No.
882396 and grant agreement No. 101078075). Views and opinions expressed are
however those of the author(s) only and do not necessarily
reflect those of the European Union or the European Research Council. Neither the European Union nor the granting authority can be held responsible for them. 
This work received additional support from the Israel Science Foundation (ISF, grant numbers 993/17 and 2549/19), Tel Aviv University Center for AI
and Data Science (TAD), the Yandex Initiative for Machine
Learning at Tel Aviv University, the Len Blavatnik and the
Blavatnik Family Foundation, and by the Israeli VATAT data science scholarship.


\bibliography{bibliography}


\clearpage
\appendix
\onecolumn

\section{Expected Regret}
In this section, we prove the following anytime expected regret guarantee for \cref{alg:batch-ensemble} (proof at the end of the section).
\begin{theorem}
\label{theorem:expected-regret}
    Suppose we run \cref{alg:batch-ensemble} with the number of batches $\nb_t = 8 \log t$. Then the following regret bounds hold for all $t \ge 1$:
    \begin{align*}
        \EE\brk[s]{\regret[t]}
        &
        \le
        \sum_{a \neq \aOpt}
        9\brk[s]*{
        \frac{8\sigma_a^2}{\Delta_a} + \frac{4}{3}
        }
        \log^2 (72t)
        \le
        \sum_{a \neq \aOpt}
        9\brk[s]*{
        \frac{8\muOpt}{\Delta_a} + \frac{28}{3}
        }
        \log^2 (72t)
        \\
        \EE\brk[s]{\regret[t]}
        &
        \le
        17 \sqrt{t \cdot\min\brk[c]*{\muOpt K, \sum_{a \neq \aOpt}\sigma_a^2}} \log(72t)
        +
        84 K \log^2(72t)
        .
    \end{align*}
\end{theorem}
Before proving \cref{theorem:expected-regret}, we first need to extend the definition of our mean estimator (\cref{eq:optimistic-mean-estimator}) to include an index for the time-varying number of batches $\nb_t$.
\paragraph{Extended notation for the mean estimator.}
Suppose we have observed $n \ge 0$ samples of an arm $a \in [K]$, i.e., $\lna[n'], n' \in [n]$. We build the following mean estimator at time $t \ge 1$. First, let, $\nb_t \ge 1$ be a batch number to be determined later. Next, we split the $n$ samples of arm $a$ into $\nb_t$ (near-)equal batches
\begin{align*}
    \batchIdxs_{t, n, a, \nb'}
    =
    \brk[c]*{n' : n' = \nb' + i \cdot \nb_t \le n, i \in \ZZ_{\ge 0}}
    ,
    \nb' \in [\nb_t]
    .
\end{align*}
Our batch ensemble estimator is 
\begin{align}
\label{eq:optimistic-mean-estimator-aug}
    \muHat_{t,n,a}
    =
    \min_{\nb' \in [\nb_t]} \muHat_{t,n,a,\nb'}
    ,
    \text{ where }
    \muHat_{t,n,a,\nb'}
    =
    \sum_{n' \in \batchIdxs_{t,n,a,\nb'}} \frac{\lna[n']}{\abs{\batchIdxs_{t,n,a,\nb'}}+2}
    ,
\end{align}
The next result adapts \cref{lemma:mean-estimator-concentration} to the above form.
\begin{lemma}
\label{lemma:mean-estimator-concentration-2}
    For any $s \ge 1, a \in [K], \nb' \in [\nb_s]$ and $n \ge \nb_s(1 + (4/\Delta_a))$
    \begin{align*}
        \Pr\brk{\muHat_{s,n,a} - \mu_a < -\Delta_a}
        \le
        \nb_s
        e^{-\frac{\Delta_a^2((n/\nb_s) - 1)}{8 \sigma_a^2 + (4 \Delta_a/3)}}
    \end{align*}
\end{lemma}
\begin{proof}
Recall that the Bernstein inequality in \cref{lemma:bernstein} may be written as for all $t \ge 0$
\begin{align}
\label{eq:bernstein2}
    \Pr\brk*{
        \muBar - \mu
        \le
        - \epsilon
    }
    \le
    e^{-\frac{n\epsilon^2}{2 \sigma^2 + (2 \epsilon/3)}}
    .
\end{align}
For ease of notation, let $m = \abs{\batchIdxs_{n,a,\nb'}} \ge (n / \nb_s) - 1 \ge 4 / \Delta_a$ where the last transition used the lower bound on $n$. Thus, we have
\begin{align*}
    \Pr\brk{\muHat_{s,n,a,\nb'} - \mu_a < -\Delta_a}
    &
    =
    \Pr\brk*{\frac{1}{m}\sum_{n \in \batchIdxs_{n,a,\nb'}} (\lna - \mu_a) < -\brk*{\frac{m+2}{m}\Delta_a - \frac{2 \mu_a}{m}}}
    \\
    &
    =
    \Pr\brk*{\frac{1}{m}\sum_{n \in \batchIdxs_{n,a,\nb'}} (\lna - \mu_a) < -\brk*{\Delta_a - \frac{2 \muOpt}{m}}}
    \\
    \tag{$2/m \le \Delta_a / 2$}
    &
    \le
    \Pr\brk*{\frac{1}{m}\sum_{n \in \batchIdxs_{n,a,\nb'}} (\lna - \mu_a) < -\brk*{\Delta_a / 2}}
    \\
    \tag{\cref{eq:bernstein2}}
    &
    \le
    e^{-\frac{m\Delta_a^2}{8 \sigma_a^2 + (4 \Delta_a/3)}}
    \\
    &
    \le
    e^{-\frac{\Delta_a^2((n/\nb_s) - 1)}{8 \sigma_a^2 + (4 \Delta_a/3)}}
    .
\end{align*}
We conclude that
\begin{align*}
    \Pr\brk{\muHat_{s,n,a} - \mu_a < -\Delta_a}
    =
    \Pr\brk{\min_{\nb' \in [\nb_s]} \muHat_{s,n,a,\nb'} - \mu_a < -\Delta_a}
    \le
    \sum_{\nb' \in [\nb_s]}
    \Pr\brk{\muHat_{s,n,a,\nb'} - \mu_a < -\Delta_a}
    \le
    \nb_s
    e^{-\frac{\Delta_a^2((n/\nb_s) - 1)}{8 \sigma_a^2 + (4 \Delta_a/3)}}
    ,
\end{align*}
where the first inequality used the union bound.
\end{proof}

Next, we need the following restatement of a result by \cite{audibert2009exploration}, which bounds the expected number of sub-optimal arm pulls for any index policy.
\begin{lemma}[\cite{audibert2009exploration}, Theorem 2]
\label{lemma:nta-decomp}
    For any integers $u >1, t \ge 1$ and $i \in [K], $ we have that
    \begin{align*}
        \EE\brk[s]{\nta}
        &
        \le
        u
        +
        \sum_{s=u+K-1}^{t-1} \sum_{n=u}^{s-1}
        \Pr\brk{\muHat_{s,n,a} - \mu_a < -\Delta_a}
        +
        \Pr\brk{\exists n \in [s-1] \text{ s.t } \muHat_{s,n,\aOpt} \ge \muOpt}
        .
    \end{align*}
\end{lemma}

Before concluding the proof of \cref{theorem:expected-regret}, we prove the following bound on the expected number of sub-optimal arm pulls.
\begin{lemma}
\label{lemma:nta-bound}
    Suppose we run \cref{alg:batch-ensemble} with $\nb_t = 8 \log t$, then
    \begin{align*}
        \EE\brk[s]{\nta}
        \le
        9\brk[s]*{
        \frac{8\sigma_a^2}{\Delta_a^2} + \frac{4}{3\Delta_a}
        }
        \log^2 (72t)
        \le
        9\brk[s]*{
        \frac{8\muOpt}{\Delta_a^2} + \frac{28}{3\Delta_a}
        }
        \log^2 (72t)
        .
    \end{align*}
\end{lemma}
\begin{proof}
We bound the terms in \cref{lemma:nta-decomp} to conclude the proof. Let $u = \nb_t \brk*{1 + \brk[s]*{\frac{8 \sigma_a^2}{\Delta_a^2} + \frac{4}{3\Delta_a}}\log(72 t)}$. Then we have that
\begin{align*}
    \sum_{s=\ceil{u}+K-1}^{t-1}
    \Pr\brk{\exists n \in [s-1] \text{ s.t } \muHat_{s,n,\aOpt} \ge \muOpt}
    &
    \tag{union bound}
    \le
    \sum_{s=\ceil{u}+K-1}^{t-1}
    \sum_{n \in [s-1]}
    \Pr\brk{\muHat_{s,n,\aOpt} \ge \muOpt}
    \\
    \tag{\cref{lemma:mean-estimator-optimism}}
    &
    \le
    \sum_{s=\ceil{u}+K-1}^{t-1}
    \sum_{n \in [s-1]}
    e^{- 2\nb_s/ 7}
    \\
    &
    \le
    \sum_{s=\ceil{u}+K-1}^{t-1}
    s e^{- 2 \nb_s / 7}
    \\
    \tag{$\nb_s = 8 \log s$}
    &
    \le
    \sum_{s=\ceil{u}+K-1}^{t-1} s^{-9/7}
    \\
    \tag{$u \ge 3$}
    &
    \le
    3
    .
\end{align*}
Next, because $n \ge \ceil{u}$ satisfies the condition for \cref{lemma:mean-estimator-concentration-2}, we get that
\begin{align*}
    \sum_{n=\ceil{u}}^{s-1}
    \Pr\brk{\muHat_{s,n,a} - \mu_a < -\Delta_a}
    \tag{\cref{lemma:mean-estimator-concentration-2}}
    &
    \le
    \nb_s
    \sum_{n=\ceil{u}}^{s-1}
    e^{-\frac{\Delta_a^2((n/\nb_s) - 1)}{8 \sigma_a^2 + (4 \Delta_a/3)}}
    \\
    &
    \le
    \nb_s
    \sum_{n=\ceil{u}}^{\infty}
    e^{-\frac{\Delta_a^2((n/\nb_s) - 1)}{8 \sigma_a^2 + (4 \Delta_a/3)}}
    \\
    \tag{sum of geometric series, $\ceil{u} \ge u$}
    &
    =
    \nb_s \frac{
    e^{-\frac{\Delta_a^2((u/\nb_s) - 1)}{8 \sigma_a^2 + (4 \Delta_a/3)}}
    }
    {1 - e^{-\frac{\Delta_a^2 / \nb_s}{8\sigma_a^2 + (4\Delta_a / 3)}}}
    \\
    &
    \le
    \frac{10}{9}\nb_s 
    \frac{8\sigma_a^2 + (4\Delta_a / 3)}{\Delta_a^2 / \nb_s}
    e^{-\frac{\Delta_a^2((u/\nb_s) - 1)}{8 \sigma_a^2 + (4 \Delta_a/3)}}
    \\
    &
    =
    \frac{10}{9}\nb_s^2\brk[s]*{
    \frac{8\sigma_a^2}{\Delta_a^2} + \frac{4}{3\Delta_a}
    }
    e^{-\frac{\Delta_a^2((u/\nb_s) - 1)}{8 \sigma_a^2 + (4 \Delta_a/3)}}
    ,
\end{align*}
where the last inequality used that $1 - e^{-x} \ge 0.9 x$ for $x \in [0,3/32]$ where $x = \frac{\Delta_a^2 / \nb_s}{8\sigma_a^2 + (4\Delta_a / 3)}$ satisfies the requirements since $s \ge u \ge 3$.
We thus have that
\begin{align*}
    \sum_{s=\ceil{u}+K-1}^{t-1} \sum_{n=\ceil{u}}^{s-1}
    \Pr\brk{\muHat_{s,n,a} - \mu_a < -\Delta_a}
    &
    \le
    \sum_{s=\ceil{u}+K-1}^{t-1} 
    \frac{10}{9}\nb_s^2\brk[s]*{
    \frac{8\sigma_a^2}{\Delta_a^2} + \frac{4}{3\Delta_a}
    }
    e^{-\frac{\Delta_a^2((u/\nb_s) - 1)}{8 \sigma_a^2 + (4 \Delta_a/3)}}
    \\
    &
    \le
    72 t \log^2 t\brk[s]*{
    \frac{8\sigma_a^2}{\Delta_a^2} + \frac{4}{3\Delta_a}
    }
    e^{-\frac{\Delta_a^2((u/\nb_t) - 1)}{8 \sigma_a^2 + (4 \Delta_a/3)}}
    \\
    \tag{$(u/\nb_t) - 1 = \brk[s]*{\frac{8 \sigma_a^2}{\Delta_a^2} + \frac{4}{3\Delta_a}}\log(72 t)$}
    &
    \le
    \brk[s]*{
    \frac{8\sigma_a^2}{\Delta_a^2} + \frac{4}{3\Delta_a}
    }
    \log^2 t
    .
\end{align*}
Plugging everything into \cref{lemma:nta-decomp}, we conclude that
\begin{align*}
    \EE\brk[s]{\nta}
    &
    \le
    3
    +
    \ceil{
    \lt \brk*{
    1
    +
    \brk[s]*{\frac{8 \sigma_a^2}{\Delta_a^2} + \frac{4}{3\Delta_a}}\log(72 t)
    }
    }
    +
    \brk[s]*{
    \frac{8\sigma_a^2}{\Delta_a^2} + \frac{4}{3\Delta_a}
    }
    \log^2 t
    \\
    &
    \le
    9\brk[s]*{
    \frac{8\sigma_a^2}{\Delta_a^2} + \frac{4}{3\Delta_a}
    }
    \log^2 (72t)
    \\
    &
    \le
    9\brk[s]*{
    \frac{8\muOpt}{\Delta_a^2} + \frac{28}{3\Delta_a}
    }
    \log^2 (72t)
    ,
\end{align*}
where the last transition also used that $\sigma_a^2 \le \mu_a = \Delta_a + \muOpt$.
\end{proof}

We are now ready to prove \cref{theorem:expected-regret}.
\begin{proof}[of \cref{theorem:expected-regret}]
First, plugging \cref{lemma:nta-bound} into \cref{eq:reg-decomp} concludes the instance-dependent regret bounds.
Next, for instance-independent bounds, we use the standard method of splitting the bound according to the sub-optimality to get that for any $c > 0$,
\begin{align*}
    \EE\brk[s]{\nta} \Delta_a
    &
    =
    \EE\brk[s]{\nta} \Delta_a \brk[s]*{
        \indEvent{\Delta_a \le c^{-1}}
        +
        \indEvent{\Delta_a^{-1} < c}
    }
    \\
    &
    \le
    \frac{\EE\brk[s]{\nta}}{c}
    +
    9\brk[s]*{
    8\muOpt c + \frac{28}{3}
    }
    \log^2 (72t)
    .
\end{align*}
Plugging into \cref{eq:reg-decomp} and setting $c = \sqrt{\frac{t}{72 \muOpt K \log^2(72t)}}$ we have
\begin{align*}
    \EE\brk[s]{\regret[t]}
    &
    \le
    \sum_{a \neq \aOpt}
    \frac{\EE\brk[s]{\nta}}{c}
    +
    9\brk[s]*{
    8\muOpt c + \frac{28}{3}
    }
    \log^2 (72t)
    \\
    &
    \le
    \frac{t}{c}
    +
    9K\brk[s]*{
    8\muOpt c + \frac{28}{3}
    }
    \log^2 (72t)
    \\
    &
    \le
    17\sqrt{\muOpt t K} \log(72t)
    +
    84 K \log^2(72t)
    .
\end{align*}
Finally, we perform a similar procedure for the variance-dependent bound to get that
\begin{align*}
    \EE\brk[s]{\nta} \Delta_a
    \le
    \frac{\EE\brk[s]{\nta}}{c}
    +
    9
    \brk*{
        8\sigma_a^2 c
        +
        \frac{4}{3}
    }\log^2(72t)
    ,
\end{align*}
and thus setting $c = \sqrt{\frac{t}{72 (\sum_{a \neq \aOpt} \sigma_a^2) \log^2(72t)}}$, we have
\begin{align*}
    \EE\brk[s]{\regret[t]}
    &
    \le
    \sum_{a \neq \aOpt}
    \frac{\EE\brk[s]{\nta}}{c}
    +
    9
    \brk*{
        8\sigma_a^2 c
        +
        \frac{4}{3}
    }\log^2(72t)
    \\
    &
    \le
    \frac{t}{c}
    +
    9
    \brk*{
        8c \brk*{\sum_{a \neq \aOpt}\sigma_a^2}
        +
        \frac{4K}{3}
    }\log^2(72t)
    \\
    &
    \le
    17 \sqrt{t \sum_{a \neq \aOpt}\sigma_a^2} \log(72t)
    +
    12 K \log^2(72t)
    .
   \qedhere
\end{align*}

%
%
\end{proof}

\end{document}